\documentclass[a4paper,12pt]{article}

\usepackage{amscd}
\usepackage{amsfonts}
\usepackage{amsmath}
\usepackage{amssymb}
\usepackage{amsthm}
\usepackage[T1]{fontenc} % changing encode
\usepackage{here} % [h] for tables
\usepackage{mathrsfs} % \mathscr
\usepackage{txfonts} % colneqq
\usepackage[all]{xy} % xypic
\usepackage{algorithm} % pseudocode
\usepackage{algpseudocode} %pseudocode

\allowdisplaybreaks

\theoremstyle{plain}
\newtheorem{thm}{Theorem}[section]

\newtheorem{prp}[thm]{Proposition}

\theoremstyle{definition}

\newcommand{\vs}[1][0.2]{\vspace{#1in}\noindent\ignorespaces}
\newcommand{\ba}{\begin{array*}}
\newcommand{\ea}{\end{array*}}
\newcommand{\be}{\begin{eqnarray*}}
\newcommand{\ee}{\end{eqnarray*}}
\newcommand{\bi}{\begin{itemize}}
\newcommand{\ei}{\end{itemize}}
\newcommand{\bb}{\vs\begin{itembox}}
\newcommand{\eb}{\end{itembox}}
\newcommand{\bc}{\begin{center}}
\newcommand{\ec}{\end{center}}
\newcommand{\bs}{\vs\begin{screen}}
\newcommand{\es}{\end{screen}}

\def\ens#1{{\mathchoice{\left\{ #1 \right\}}{\{ #1 \}}{\{ #1 \}}{\{ #1 \}}}}
\def\set#1#2{{\mathchoice{\left\{ #1 \middle| #2 \right\}}{\{ #1 \mid #2 \}}{\{ #1 \mid #2 \}}{\{ #1 \mid #2 \}}}}
\def\r#1{\text{\rm #1}}

\def\Bigv#1{\left| #1 \right|}
\def\v#1{{\mathchoice{\Bigv{#1}}{| #1 |}{| #1 |}{| #1 |}}}
\def\Bign#1{\left\| #1 \right\|}
\def\n#1{{\mathchoice{\Bign{#1}}{\| #1 \|}{\| #1 \|}{\| #1 \|}}}

\newcommand{\bG}{\mathbb{G}}

\newcommand{\bN}{\mathbb{N}}

\newcommand{\bR}{\mathbb{R}}

\newcommand{\bZ}{\mathbb{Z}}

\newcommand{\cO}{\mathscr{O}}

\newcommand{\rC}{\r{C}}

\newcommand{\rH}{\r{H}}

\newcommand{\N}{\bN}

\newcommand{\R}{\bR}
\newcommand{\Z}{\bZ}

\newcommand{\Gm}{\mathbb{G}_{\r{m}}}
\newcommand{\Qp}{\mathbb{Q}_p}

\newcommand{\Zp}{\mathbb{Z}_p}

\newcommand{\im}{\r{im}}

\newcommand{\CO}{\r{CO}}
\newcommand{\hGm}{\hat{\bG}_{\r{m}}}

\algnewcommand\algorithmicbreak{{\bf break}}
\algnewcommand\Break{\algorithmicbreak{}}
\algnewcommand\algorithmiccontinue{{\bf continue}}
\algnewcommand\Continue{\algorithmiccontinue{}}

\title{$p$-adic Manifold Learning and Benchmark Tasks from Impartial Games}
\author{Tomoki Mihara}
\date{}

\begin{document}

\maketitle
%\address
\begin{abstract}
We introduce $p$-adic manifold learning, propose an algorithm to solve it, and propose benchmark tasks from impartial games.
\end{abstract}

\tableofcontents
%\fn{}{}

\section{Introduction}
\label{Introduction}

Given an impartial game $G$ whose positions are labeled by a vector of natural numbers, how its Grundy number behaves as a function? Of course, the answer heavily depends on $G$ and the labeling. However, it is excellent if there is a way to estimate the function from small sample data and an appropriate labeling.

\vs
Nim, the simplest impartial game, has been completely analysed in the sense that its Grundy number is known to be computed by the nimber sum. Since the nimber sum is $2$-adically continuous, it is reasonable to expect that the set of winning positions can be interpolated from the data of winning positions with small piles through some $2$-adic optimisation method. However, what is $2$-adic optimisation?

\vs
Let $p$ be a prime number. The history of $p$-adic optimisation is much shorter than that of real optimisation. S.\ Albeverio, A.\ Khrennikov, and B.\ Tirrozi studied $p$-adic neural network in \cite{AKT99} and \cite{KT00}, P.\ E.\ Bradley studied dendrograms and clusterings using $p$-adic numbers in \cite{Bra08} and \cite{Bra09}, and so on. Introduction of \cite{Bra25} explains in detail the history around applications of the $p$-adic numbers. Of course, even in these years, there are still many investigations in $p$-adic optimisation (cf.\ \cite{ZZ23}, \cite{ZZB24}, \cite{BMP25}, \cite{Zub25-1}, \cite{Zub25-2}, \cite{Ngu25}, \cite{Mih26-1}, \cite{Mih26-2}, \cite{Mih26-3}, \cite{Mih26-4}).

\vs
See Introduction of \cite{Mih26-1} for the similarity and the difference of real optimisation and $p$-adic optimisation. For the reader's convenience, we give a duplicated explanation here. There are several real optimisation methods which can be easily modified to be $p$-adic optimisation methods. For example, linear algebraic methods independent of the coefficient field works for a general field, and hence for the $p$-adic number field. Combinatorial methods such as greedy algorithm, hill climbing, and simulated annealing also work for $p$-adic optimisation, as long as we modify the distance from the $p$-adic distance to another similar distance such as a variant of Hamming distance on $p$-adic expansions. We note that since the non-Archimedean property prevents reaching a far point by repetition of $p$-adically small steps, optimisation based on repetition of small steps sometimes requires to use another distance than the $p$-adic distance.

\vs
On the other hand, the lack of $p$-adic counterparts of real methods based on gradients prevents from inventing $p$-adic counterparts of real optimisation methods. Although Newton's method also works for multivariable polynomials in $\Qp$, it is not applicable to an optimisation problem of a general $p$-adic function, which does not necessarily have a zero.

\vs
Another problem is that when we deal with a loss function $\epsilon$ of the form $X \to [0,\infty)$ for a subset $X$ of $\Qp$, the differential of $\epsilon$ does not naturally make sense. In order to differentiate a function $f$, we need to assume that the domain of $f$ and the codomain of $f$ admit ambient spaces with a compatible arithmetic structure. In the setting of $\epsilon$, the ambient space of $X$ is $\Qp$, while that of $[0,\infty)$ is $\R$. Since we do not have a definition of arithmetic of a pair of a $p$-adic number and a real number, the usual formulation fails.

\vs
In addition, even if we extend the notion of a differential so that it vanishes at any point where $\epsilon$ is locally constant, it frequently occurs that $\epsilon$ is locally constant at almost all points of the domain of $\epsilon$. Therefore, it is very difficult to make use of a differential to determine a small step in optimisation.

\vs
We go back to the original question. How can we estimate the set of winning positions from small sample values, under an assumption that the Grundy number defines a $2$-adically continuous function? Generalising this question, we consider the following problem: Let $X$ be a topological space, $Y$ a subset of $X$, $R$ a topological ring, $f \colon X \to R$ a continuous map with $Y = f^{-1}(0)$, and $\vec{y}$ finite sample data of elements of $Y$. Estimate $Y$ from the data of $(X,R,\vec{y})$.

\vs
In the nim setting, $X$ is $\Z_2^D$, where $D$ is the fixed number of heaps, $Y$ is the closure of the set of winning positions regarded as elements of $\N^D$ through the natural labeling, $R$ is $\Z_2$, $f$ is the unique continuous extension of the function $\N^D \to \N$ give by the Grundy number, and $\vec{y}$ is finite sample data of losing positions.

\vs
On the other hand, if $X$ is a Euclidean space, $Y$ is an algebraic subvariety of $X$, $R$ is $\R$, and $f$ is a polynomial function with coefficient in $\R$, then the problem is a part of what is called {\it manifold learning}. In this sense, we are considering a variant of manifold learning. Especially when $X$ and $R$ are $p$-adic objects, the problem should be treated as a $p$-adic counterpart of manifold learning.

\vs
We note that manifold learning mainly aims at dimensionality reduction for a practical reason. Even when we do not consider dimensionality reduction, we usually consider the case that small noise occurs in the observation of $\vec{y}$. Otherwise, $f$ is decisively estimated by Gauss' elimination applied to Vandermonde matrix.

\vs
On the other hand, we do not consider the dimensionality reduction in $p$-adic manifold learning, because Brouwer's theorem implies that $\Zp^D$ is homeomorphic to $\Zp$ for any $D \in \N_{> 0}$. In addition, we do not assume that $f$ is given as a polynomial function. Therefore, the algebraic approach such as the decisive polynomial regression by Gauss' elimination applied to Vandermonde matrix does not work here.

\vs
Here is a summary of our strategy for $p$-adic manifold learning for the case $X = \Zp^D$ with $D \in \N_{> 0}$, $Y$ is a closed subspace of $X$, and $R = \Zp$.
\bi
\item[(1)] Estimate the defining function $f$ of $Y$ modulo $p^E$ for a fixed hyper-parameter $E \in \N$ by solving nearest neighbour searching problems using a $p$-adic counterpart of kd-tree applied to the image of $\vec{y}$. We note that $f$ is not necessarily unique, and hence the estimation does not necessarily approximate the true $f$.
\item[(2)] Apply finite rank approximation to the estimation of $f$ modulo $p^E$ by higher dimensional Mahler expansion in order to avoid overfitting the sparse sample data $\vec{y}$. Then $Y$ is estimated as the zero locus of the resulting function.
\ei
We note that we extend this method to the case where $X$ is a profinite Abelian group through the analogy between Fourier series expansion and Mahler expansion. Although our purpose is to apply this algorithm to estimate an unknown $Y \subset X$ such as the set of winning positions of a complicated impartial game, we use $3$-heap nim in order to provide a benchmark for $p$-adic manifold learning. We note that for a general impartial game with a general labeling, we cannot expect that the Grundy number gives a $2$-adically continuous function. Therefore, we practically need to construct an appropriate labeling for which the resulting function or another specific function characterising winning positions is $p$-adically complete for some explicit prime number $p$.

\vs
We briefly explain contents of this paper. In \S \ref{Convention and Preliminaries}, we briefly recall $p$-adic analysis. In \S \ref{Zero Detection}, we recall a zero set, introduce a $\CO_{\delta}$ set as a $p$-adic counterpart of a zero set, and explain how to estimate a $\CO_{\delta}$ set from sample data. In \S \ref{p-adic Estimation of Zero Locus}, we explain the algorithm to solve $p$-adic manifold learning for $X = \Zp^D$ for a $D \in \N$, and exhibit its performance against the benchmark task given by $3$-heap nim. In \S \ref{Profinite Abelian Group Extension}, we explain the analogy between Fourier series expansion and Mahler expansion, and extend the algorithm to the case where $X$ is a profinite Abelian group.

\section{Convention and Preliminaries}
\label{Convention and Preliminaries}

We denote by $\N$ the set of non-negative integers. For a $d \in \Z$, we set $\N_{> d} \coloneqq \N \cap (d,\infty)$, and $\N_{< d} \coloneqq \N \cap [0,d)$. For sets $X$ and $Y$, we denote by $T^S$ the set of maps $S \to T$. We note that every $d \in \N$ is identified with $\N_{< d}$ in set theory, and hence $T^d$ formally means $T^{\N_{< d}}$, which is naturally identified with the set of $d$-tuples in $T$.

\vs
When we write a pseudocode, a for-loop along a subset of $\N$ denotes the loop of the ascending order, and a for-loop along $I$ denotes a loop in an arbitrary order.

\vs
We recall basic notions in $p$-adic number theory. Throughout the paper, $\sup$ denotes the supremum in $[0, \infty]$. In particular, the supremum of the empty set is $0$.

\vs
A {\it Banach $\Qp$-vector space} is a $\Qp$-vector space $V$ equipped with a map $\n{\cdot} \colon V \to [0,\infty)$ called {\it the norm} satisfying the following:
\begin{itemize}
\item[(i)] The map $V \times V \to [0,\infty), \ (v_0,v_1) \mapsto \n{v_0 - v_1}$ is a complete ultrametric.
\item[(ii)] For any $(c,v) \in \Qp \times V$, $\n{cv} = \v{c} \ \n{v}$ holds.
\end{itemize}
For example, for any set $I$, the set $\rC_0(I,\Qp)$ of sequences $(c_i)_{i \in I} \in \Qp^I$ in $\Qp$ indexed by $I$ converging to $0$ forms a Banach $\Qp$-vector space with respect to the supremum norm. Here, a sequence $(c_i)_{i \in I} \in \Qp^I$ indexed by a set $I$ is said to {\it converge to $0$} if $\# \set{i \in I}{\v{c_i} \geq \epsilon} < \infty$ for any $\epsilon \in (0,\infty)$.

\vs
An {\it orthonormal Schauder basis} of a Banach $\Qp$-vector space $V$ is a sequence $(F_i)_{i \in I}$ of elements of $V$ indexed by a set $I$ satisfying the following for any $f \in V$:
\bi
\item[(1)] There uniquely exists a $(c_i)_{i \in I} \in \rC_0(I,V)$ such that
\be
\sum_{i \in I} c_i F_i \coloneqq \lim_{I'} \sum_{i \in I'} c_i F_i
\ee
converges to $f$, where $I'$ runs through all finite subsets of $I$ and the limit is taken along the directed partial order given by the inclusion.
\item[(2)] The unique $(c_i)_{i \in I}$ in (1) satisfies $\n{f} = \sup_{i \in I} \v{c_i}$.
\ei
For example, the sequence $(\delta_i)_{i \in I}$ of delta functions
\be
\delta_i \colon I & \mapsto & \Qp \\
j & \mapsto &
\left\{
\begin{array}{ll}
1 & (j = i) \\
0 & (j \neq i)
\end{array}
\right.
\ee
forms an orthonormal Schauder basis of $\rC_0(I,\Qp)$ for any set $I$.

\vs
Let $X$ be a compact topological space. Then the set $\rC(X,\Qp)$ of continuous functions $X \to \Qp$ forms a Banach $\Qp$-vector space with respect to the supremum norm, because the compactness of $X$ implies that every continuous function $X \to \Qp$ is bounded. We are mainly interested in the case where $X$ is ultrametrisable, i.e.\ there is an ultrametric on $X$ compatible with the topology of $X$.

\vs
An {\it ultrametric} on $X$ is a map $d \colon X^2 \to [0,\infty)$ satisfying the following:
\bi
\item[(1)] For any $(x,y) \in X^2$, $d(x,y) = 0$ is equivalent to $x = y$.
\item[(2)] For any $(x,y) \in X^2$, the equality $d(x,y) = d(y,x)$ holds.
\item[(3)] For any $(x,y,z) \in X^3$, the inequality $d(x,z) \leq \max \ens{d(x,y),d(y,z)}$ holds.
\ei
In particular, the condition (3) implies the triangular inequality, and hence every ultrametric is a metric. A metric $d$ on $X$ is said to be {\it compatible with the topology of $X$} if the metric topology associated to $d$ coincides with the topology of $X$.

\vs
If $X$ is totally disconnected and Hausdorff, then $X$ admits an orthonormal Schauder basis consisting of characteristic functions of clopen subsets by \cite{BM24} Proposition 7.1. In addition, if $X$ is ultrametrisable, then such an orthonormal Schauder basis can be taken in an explicit way using a compatible ultrametric, and forms a generalised van der Put basis by \cite{BM24} Proposition 7.5 (i) and Proposition 7.7. Here, a generalised van der Put basis is a generalisation of van der Put basis, which is an orthonormal Schauder basis of $\rC(\Zp,\Qp)$ (cf.\ \cite{Mah73} Chapter 16 \S 9 Theorem 4). Since we do not focus on a generalised van der Put basis in this paper, we omit the detail.

\vs
Instead, we focus on another orthonormal Schauder basis called {\it Mahler basis} (cf.\ \cite{Mah58} Theorem 1) . It is an orthonormal Schauder basis of $\rC(\Zp,\Qp)$ given as $(\binom{x}{\ell})_{\ell=0}^{\infty}$, where $\binom{x}{\ell}$ denotes the $\ell$-th binomial coefficient function for an $\ell \in \N$.
\be
\binom{x}{0} & = & 1 \\
\binom{x}{1} & = & x \\
\binom{x}{2} & = & \frac{x(x-1)}{2} \\
\binom{x}{3} & = & \frac{x(x-1)(x-2)}{6}
\ee
Mahler expansion $(c_{\ell})_{\ell=0}^{\infty} \in \rC_0(\N,\Qp)$ of an $f \in \rC(\Zp,\Qp)$, i.e.\ the coefficient of the expression of $f$ with respect to Mahler basis, is given by the explicit formula
\be
c_i = \sum_{j=0}^{i} (-1)^{i-j} \binom{i}{j} f(j),
\ee
and can be computed in many ways, e.g.\ higher order difference operators, Gauss' elimination, inclusion-exclusion principle or M\"ovius inversion for
a $p$-adic functional on a power set, and so on.

\vs
Let $D \in \N$. The $D$-dimensional dimensional Mahler basis is an orthonormal Schauder basis of $\rC(\Zp^D,\Qp)$ given as $(\prod_{d=0}^{D-1} \binom{x_d}{\ell_d})_{(\ell_d)_{d=0}^{D-1} \in \N^D}$ (cf.\ \cite{Mah80} \S 12 for the $2$-dimensional case). We note that the higher dimensional case is immediately reduced to the $1$-dimensional case by $p$-adic monoidal Gel'fand Naimark duality (cf.\ \cite{Mih21} Corollary 2.5). Namely, for any $(d_0,d_1) \in \N^2$, the multiplication
\be
\rC(\Zp^{d_0},\Qp) \times \rC(\Zp^{d_1},\Qp) \to \rC(\Zp^{d_0+d_1},\Qp)
\ee
induces an isometric $\Qp$-algebra isomorphism
\be
\rC(\Zp^{d_0},\Qp) \hat{\otimes}_{\Qp} \rC(\Zp^{d_1},\Qp) \to \rC(\Zp^{d_0+d_1},\Qp).
\ee
In particular, the higher dimensional Mahler expansion is computed as the repetition of the $1$-dimensional Mahler expansion. Here is a pseudocode of the higher dimensional Mahler expansion based on in-place recursion of the difference operator:

\begin{figure}[H]
\begin{algorithm}[H]
\caption{$D$-dimensional Mahler expansion of a continuous function $F \colon \Zp^D \to \Zp$ truncated at the $L$-th entry modulo $p^E$ computed from the multi-dimensional array $\vec{f} = (\vec{f}_x)_{x=0}^{L-1}$ of values of $F$ modulo $p^E$ on $\N_{< L}^D$}
\label{Mahler}
\begin{algorithmic}[1]
\Function {Mahler}{$p,E,D,L,\vec{f}$}
	\If {$D = 0$}
		\State \Return $\vec{f}$
	\EndIf
	\ForAll {$k \in \N_{< L}$}
		\ForAll {$x \in \N \cap (k,L)$ in the reverse order}
			\State $\vec{f}_x \gets$ \Call{Mahler}{$p,E,D-1,L,(\vec{f}_x - \vec{f}_{x-1}) \bmod p^E$}
		\EndFor
	\EndFor
	\State \Return $\vec{f}$
\EndFunction
\end{algorithmic}
\end{algorithm}
\end{figure}

In order to compute $F(\vec{x})$ modulo $p^E$ for a vector $\vec{x}$ of natural numbers with a fixed upperbound $M$ of entries, it suffices to compute all binomial coefficients modulo $p^E$ with arguments in $\N_{\leq M}$ using Pascal's triangle as a preprocess, and naively apply the definition of Mahler expansion modulo $p^E$.

\section{Zero Detection}
\label{Zero Detection}

Let $X$ be a topological space, and $Y$ a subset of $X$. We consider the following problem: Given sample points from $Y$, determine $Y$, or determine whether a given point $x \in X$ belongs to $Y$ or not.

\vs
We call this problem the {\it subset membership detection problem} for $(X,Y)$. Of course, a subset membership detection problem cannot be solved unless we additionally assume strong conditions on $Y$ and sample points.

\subsection{Zero Set}
\label{Zero Set}

In order to explain conditions, we first recall real analysis. The subset $Y$ is said to be a {\it zero set} of $X$ if there exists a continuous function $f \colon X \to \R$ such that $Y = f^{-1}(0)$, and is said to be a {\it $G_{\delta}$ set} of $X$ if it is a countable intersection of open subsets of $X$. Since $\ens{0} \subset \R$ is a closed $G_{\delta}$ set, every zero set is a closed $G_{\delta}$ set. When $X$ is metrisable, every closed subset is a zero set, because the distance from the closed subset gives a defining function.

\vs
It is notable that a finite intersection of zero sets is again a zero set, because for any finite set $S$ of continuous functions $f \colon X \to \R$, an $x \in X$ is a common zero of $S$ if and only if $x$ is a zero of $\sum_{f \in  S} f^2$. Therefore, a typical example of a zero set is the subvariety of the Euclidean space defined by finitely many real differentiable functions.

\vs
Suppose $Y$ is a zero set of $X$. Given sample points in $Y$, estimation of a continuous function defining $Y$ from the data of the sample points gives a solution of the subset membership detection problem for $(X,Y)$. Especially when $X$ is an Euclidean space and $Y$ is a subvariety, the problem to estimate a defining function of $Y$ is a part of what is called {\it manifold learning}.

\subsection{$\CO_{\delta}$ Set}
\label{CO_delta Set}

In $p$-adic analysis, we are mainly interested in $p$-adic functions rather than real functions. Therefore we consider a $p$-adic counterpart of the notion of a zero set. We say that $Y$ is a {\it $\CO_{\delta}$ set} of $X$ if it is a countable intersection of clopen subsets of $X$. It is elementary to show that $Y$ is a $\CO_{\delta}$ set of $X$ if and only if there exists a continuous function $f \colon X \to \Zp$ such that $Y = f^{-1}(0)$ (cf.\ proof of \cite{BM23} Lemma 2.9), and hence it is a $p$-adic counterpart of a zero set. We recall a relation between an ultrametric space and a $\CO_{\delta}$ set similar to that between a metric space and a zero set.

\begin{prp}
\label{closed vs CO_delta}
Suppose that $X$ is ultrametrisable. Then the following are equivalent:
\bi
\item[(1)] The subset $Y$ is closed.
\item[(2)] The subset $Y$ is a $\CO_{\delta}$ set of $X$.
\ei
\end{prp}

\begin{proof}
It suffices to show that (1) implies (2). Suppose $Y$ is closed. Take an ultrametric $d$ on $X$ compatible with its topology. Take an arbitrary sequence $(r_n)_{n \in \N}$ of positive real numbers converging to $0$, e.g.\ $((1+n)^{-1})_{n=0}^{\infty}$ or $(p^{-n})_{n=0}^{\infty}$. For each $r \in (0,\infty)$, we set
\be
B(Y,r) \coloneqq \bigcup_{y \in Y} \set{x \in X}{d(x,y) \leq r_n}.
\ee
Then we have
\be
Y = \bigcap_{n \in \N_{> 0}} B(Y,r_n)
\ee
by the closedness of $Y$. Therefore, it suffices to show that the union of closed balls of a common radius with respect to $d$ is clopen.

\vs
Let $r \in (0,\infty)$. We denote by $C$ the set of the union of closed balls of radius $r$ with respect to $d$. Since $d$ is an ultrametric, every closed ball with respect to $d$ is open, and hence every element of $C$ is open. Let $U \in C$. We show that $U$ is clopen. As we have shown, $U$ is open. Since $d$ is an ultrametric, the set of closed balls of radius $r$ with respect to $d$ is a disjoint covering of $X$. Therefore, $X \setminus U$ again belongs to $C$, and hence is open. This implies that $U$ is clopen.
\end{proof}

By the proof of Proposition \ref{closed vs CO_delta}, if $X$ is ultrametrisable by an ultrametric $d$ and $Y$ is a $\CO_{\delta}$ set of $X$, the function
\be
F_Y \colon X & \to & \Zp \\
x & \mapsto & 
\left\{
\begin{array}{ll}
0 & (x \in Y) \\
p^{\max \set{n \in \N}{x \in B(Y,p^{-n})}} & (x \in B(Y,1) \setminus Y) \\
1 & (x \in X \setminus B(Y,1))
\end{array}
\right.
\ee
is a continuous function defining $Y$. Using this specific defining function $F_Y$, the subset membership detection problem for $(X,Y)$ is reduced to estimation of $B(Y,p^{-n})$ for each $n \in \N$.

\vs
Especially when $X$ is a $p$-adic analytic space in some sense, e.g.\ the set of $\Qp$-rational points of a rigid analytic space over $\Qp$ (cf.\ \cite{Ber90}, \cite{Ber93}, \cite{BGR84}, \cite{Hub96}, and so on for various formulation of a rigid analytic space), and $Y$ is a closed subspace of $X$, then it is reasonable to call this problem the {\it $p$-adic manifold learning} for $(X,Y)$. We casually extend the reasonable terminology to a wider case: $X$ is a topological space and $Y$ is its $\CO_{\delta}$ set in order not to care about the difference of formulation of a $p$-adic analytic space. In order to solve the $p$-adic manifold learning for $(X,Y)$, it suffices to estimate $B(Y,p^{-n})$ for each $n \in \N$.

\vs
We mainly consider the case $X = \Zp^D$ for a $D \in \N$. Since $\Zp^D$ is ultrametrisable by the ultrametric associated to the $p$-adic $\ell^{\infty}$-norm $\n{\cdot}_{\infty}$, the $p$-adic manifold learning for $(\Zp^D,Y)$ makes sense. In order to estimate the function $F_Y$ for this case, we need to consider a data structure to handle a set of vectors of $p$-adic integers admitting a rapid method to answer nearest neighbour searching queries. We will introduce a $p$-adic analogue of kd-tree algorithm for this purpose in the next subsection.

\subsection{$p$-adic kd-trie}
\label{p-adic kd-trie}

Trie tree algorithm is one of the most effective data structures to handle a set of $p$-adic integers. Since $\Zp^D$ is homeomorphic to $\Zp$ for any $D \in \N _{> 0}$ by Brouwer's theorem (the characterisation theorem of the Cantor set originated from the study in \cite{Bro10}), trie tree algorithm is applicable also to a set of vectors of $p$-adic integers as long as we fix a homeomorphism, which can be explicitly chosen through the proof of Brouwer's theorem as long as we fix a system of clopen subsets.

\vs
We recall A.\ P.\ Zubarev's homeomorphism in \cite{Zub25-1}. For a $p$-adic integer $n$ and an $e \in \N$, we denote by $n[e] \in \N \cap [0,p)$ the $(1+e)$-th digit of the $p$-adic expansion of $n$. We have
\be
n = \sum_{e=0}^{\infty} n[e] p^e
\ee
for any $n \in \Zp$. For a $D \in \N_{> 0}$, the homeomorphism is given as
\be
\Zp^D & \to & \Zp \\
(n_d)_{d=0}^{D-1} & \mapsto & \sum_{d=0}^{D-1} \sum_{e=0}^{\infty} n_d[e] p^{me+d}.
\ee
See also \cite{Zub25-2} for the application of the homeomorphism to $p$-adic polynomial regression.

\vs
The composite of the homeomorphism and trie tree algorithm for $\Zp$, which we call {\it $p$-adic kd-trie algorithm} is a $p$-adic analogue of kd-tree algorithm. Similar to the original kd-tree algorithm for real vectors, the $p$-adic kd-trie algorithm for $p$-adic vectors provides an effective way to handle a set/multiset of $p$-adic integers and to answer various queries on it. For example, a nearest neighbour searching query is rapidly solved with a much easier implementation than the real one, because the non-Archimedean property makes branch-and-bound simpler. Namely, if we want to compute the maximum of the $p$-adic valuation of the difference between a $p$-adic vector $v$ and points in a given set $S$ of $p$-adic vectors, it suffices to simply check how many times moves from a node to a child node corresponding to the strings associated to $v$ can be done without adding a new node to the trie tree constructed from $S$. Here is a pseudocode of the process:

\begin{figure}[H]
\begin{algorithm}[H]
\caption{Multidimensional $p$-adic expansion}
\label{multidimensional p-adic expansion}
\begin{algorithmic}[1]
\Function {Expand}{$p,E,D,\vec{x}$}
	\State $a \gets$ an empty array
	\ForAll {$e \in \N_{< E}$}
		\ForAll {$d \in \N_{< D}$}
			\State Append $x_d \bmod p$ to $a$ \Comment{$\vec{x} = (x_d)_{d=0}^{D-1}$}
			\State $x_d \gets \lfloor p^{-1}x_d \rfloor$
		\EndFor
	\EndFor
	\State \Return $a$
\EndFunction
\end{algorithmic}
\end{algorithm}
\end{figure}

\begin{figure}[H]
\begin{algorithm}[H]
\caption{Construction of $p$-adic kd-trie}
\label{kd-trie}
\begin{algorithmic}[1]
\Function {KDTrie}{$p,E,D,S$}
	\State $T \gets$ a rooted tree with $V_T = \ens{\ast_T}$ \Comment{variable for a trie tree whose edges are labeled by $\N_{< p}$}
	\ForAll {$\vec{x} \in S$}
		\State $a = (a_i)_{i=0}^{ED-1} \gets$ \Call{Expand}{$p,E,D,\vec{x}$}
		\State $n \gets \ast_T$
		\ForAll {$i \in \N_{< ED}$}
			\If {There is no directed edge of label $a_i$ from $n$}
				\State Append to $T$ a new directed edge of label $a_i$ from $n$ to a new node
			\EndIf
			\State $n \gets$ the unique node to which there is a directed edge of label $a_i$ from $n$
		\EndFor
	\EndFor
	\State \Return $T$
\EndFunction
\end{algorithmic}
\end{algorithm}
\end{figure}

Here, for a tree $T$, we denote by $\ast_T$ the root of $T$, and by $V_T$ the set of vertices of $T$.

\begin{figure}[H]
\begin{algorithm}[H]
\caption{Searching the nearest neighbour of $\vec{x}$ in $S$ implemented as a $p$-adic kd-trie $T$}
\label{nearest neighbour searching}
\begin{algorithmic}[1]
\Function {NearestNeighbourSearching}{$p,E,D,T,\vec{x}$}
	\State $a = (a_i)_{i=0}^{ED-1} \gets$ \Call{Expand}{$p,E,D,\vec{x}$}
	\State $n \gets \ast_T$
	\ForAll {$i \in \N_{< ED}$}
		\If {There is no directed edge of label $a_i$ from $n$}
			\State \Return $\lfloor D^{-1}i \rfloor$
		\EndIf
		\State $n \gets$ the unique node to which there is a directed edge of label $a_i$ from $n$
	\EndFor
	\State \Return $\infty$
\EndFunction
\end{algorithmic}
\end{algorithm}
\end{figure}

We note that the return value of Algorithm \ref{nearest neighbour searching} is the maximum of the $p$-adic valuation rather than the minimum of the $p$-adic distance for convenience. Practically, it is good to implement $\infty$ as $E$ because other return values of Algorithm \ref{nearest neighbour searching} is smaller than $E$.

\section{$p$-adic Estimation of Zero Locus}
\label{p-adic Estimation of Zero Locus}

Let $D \in \N_{> 0}$, $Y$ a closed subset of $\Zp^D$, $I$ a finite set, and $\vec{y} = (y_i)_{i \in I} \in Y^I$ a sequence of sample points in $Y$. We estimate $Y$ by estimating $(B(Y,p^{-e}))_{e=0}^{E}$ from the sample data $\vec{y}$.

\subsection{$p$-adic Interpolation of Zero Locus}
\label{p-adic Interpolation of Zero Locus}

If $\vec{y}$ is biased in the sense that it is not taken generically in $Y$, we cannot estimate $Y$ from $\vec{y}$. If we allowed $I$ to be an infinite set, then it would be good to assume that the image $\im \vec{y} \coloneqq \set{y_i}{i \in I} \subset Y$ of $\vec{y}$ is dense in $Y$. Instead, we introduce the notion of $E$-denseness as a finite variant of denseness.

\vs
Let $E \in \N$. We say that $\vec{y}$ is {\it $E$-dense} in $Y$ if for any $y \in Y$, there exists an $i \in I$ such that $\| y - y_i \| \leq p^{-E}$. Under the assumption, for any $e \in \N_{leq E}$, we have
\be
B(Y,p^{-e}) = B(\im \vec{y},p^{-e}).
\ee
Estimation of $B(Y,p^{-e})$ for each $e \in \N_{\leq E}$ gives estimation of the defining function $F_Y$ modulo $p^E$.

\vs
Let $x \in X$. If $F_Y(x) \not\equiv 0 \pmod{p^E}$, then we have $F_Y(x) \neq 0$, and hence $x \notin Y$. On the other hand, if $F_Y(x) \equiv 0 \pmod{p^E}$, both cases are possible: $F_Y(x) = 0$ and $F_Y(x) \in p^E \Zp \setminus \ens{0}$. Therefore, we can use $F_Y$ modulo $p^E$ as estimation of how likely $x \in Y$ holds. As $E$ grows bigger, then the estimation becomes more accurate. On the other hand, if there can be $p$-adically small noise on the sample data $\vec{y}$, it is reasonable to keep $E$ smaller than or equal to the maximum of the $p$-adic valuation of the noise in order to avoid overfitting.

\vs
Suppose $X = \Zp^D$ for a $D \in \N$. The estimation of $F_Y$ modulo $p^E$ requires the data of the characteristic functions of $(B(Y,p^{-e})_{e=0}^{E} = (B(\im \vec{y},p^{-e}))_{e=0}^{E}$, which can be computed by Algorithm \ref{nearest neighbour searching} applied to $(p,E,D,T)$ for the $p$-adic kd-trie $T$ of $\im \vec{y}$ computed by Algorithm \ref{kd-trie} applied to $(p,E,D,\im \vec{y})$.

\vs
However, the assumption that $\vec{y}$ is $E$-dense in $Y$ is too strong when $p^{ED} \#(Y \cap \N_{< R}^D)/ \#(X \cap \N_{< R}^D)$ is large for a fixed upperbound $R \in \N$ of entries, because a lowerbound of $\# I$ is roughly estimated as $p^{ED} \#(Y \cap \N_{< R}^D) / \#(X \cap \N_{< R}^D)$. Especially when $Y$ is a subvariety of codimension $1$ (and hence $D > 0$), $\#(Y \cap \N_{< R}^D) / \# (X \cap \N_{< R}^D)$ is roughly estimated as $p^{-1}$, and $\# I$ should be greater than $p^{ED-1}$.

\vs
In order to solve this data sparseness issue, we will apply $p$-adic interpolation method in the next subsection.

\subsection{$p$-adic Finite Rank Approximation of Defining Function}
\label{p-adic Finite Rank Approximation of Defining Function}

We continue to consider the case $X = \Zp^D$ for a $D \in \N$. When $D = 0$, then the $p$-adic manifold learning for $(X,Y)$ is trivial. Therefore, we assume $D > 0$ in this subsection. When $\# I$ is much smaller than $p^{ED-1}$, the estimation of $F_Y$ modulo $p^E$ might overfit the sample data. Therefore, it is reasonable to apply ``finite rank approximation'' to the estimation of $F_Y$ modulo $p^E$ in order to remove the effect of overfitting the sparse data.

\vs
The Banach $\Qp$-vector space $\rC(\Zp^D,\Qp)$ of continuous functions $\Zp^D \to \Qp$ is of countable type. A finite rank approximation $f'$ of a continuous function $f \colon \Zp^D \to \Qp$ in this context is given by truncating an orthonormal Schauder basis, i.e.\ fixing an orthonormal Schauder basis $(F_{\ell})_{\ell \in \N}$ of $\rC(\Zp^D,\Qp)$ and an $L \in \N$, presenting $f$ as $\sum_{\ell \in \N} c_{\ell} F_{\ell}$ for a unique $(c_{\ell})_{\ell \in \N} \in \Qp^{\N}$ converging to $0$, and setting $f' \coloneqq \sum_{\ell=0}^{L-1} c_{\ell} F_{\ell}$.

\vs
Properties of the finite rank approximation $f \mapsto f'$ heavily depend on the choice of $(F_{\ell})_{\ell \in \N}$ and $L$, and hence we need to select them appropriately by considering what properties we desire for the finite rank approximation. When we expect that $f$ is Lipschitz continuous with an explicit bound of Lipschitz constant, then it is good to choose $(F_{\ell})_{\ell \in \N}$ as an orthonormal Schauder basis consisting of characteristic functions of closed balls whose radii are ordered in the descending order, e.g.\ van der Put basis.

\vs
On the other hand, when we expect that $f$ can be effectively interpolated by values at vectors with small natural number entries, then it is good to choose $(F_{\ell})_{\ell \in \N}$ as the higher dimensional Mahler basis linearly ordered in some way.

\vs
Suppose that we know $Y \cap (\N_{< M})^D$ for an $M \in \N$ and $\vec{y}$ is taken so that $\im \vec{y} = Y \cap (\N_{< M})^D$. Then the estimation of $F_Y$ modulo $p^E$ by $\vec{y}$ is accurate on $\N_{< M}^D$, and Algorithm \ref{Mahler} is applicable to $(p,E,D,M^D,\vec{f})$, where $\vec{f}$ denotes the multi-dimensional array of values of $F_Y$ modulo $p^E$ on $\N_{< M}^D$. The resulting estimation of the $D$-dimensional Mahler expansion of $F_Y$ modulo $p^E$ can be used as a finite rank approximation of $F_Y$. Here is a pseudocode of the estimation:

\begin{figure}[H]
\begin{algorithm}[H]
\caption{Estimation of the $D$-dimensional Mahler expansion of $F_Y$ modulo $p^E$ computed from the sample data $\vec{y}$}
\label{defining function}
\begin{algorithmic}[1]
\Function {MahlerManifoldLearning}{$p,E,D,L,I,\vec{y}$}
	\State $T \gets$ \Call{KDTrie}{$p,E,D,\im \vec{y}$}
	\State $\vec{f} = (f_{\vec{x}})_{\vec{x} \in \N_{< M}^D} \gets (0)_{\vec{x} \in \N_{< M}^D}$
	\ForAll {$\vec{x} \in \N_{< M}^D$}
		\State $v \gets$ \Call{NearestNeighbourSearching}{$p,E,D,T,\vec{x}$}
		\State $f_{\vec{x}} \gets p^v$
	\EndFor
	\State \Return \Call{Mahler}{$p,E,D,M^D,\vec{f}$}
\EndFunction
\end{algorithmic}
\end{algorithm}
\end{figure}

\subsection{Experiment by Nim Benchmark}
\label{Experiment by Nim Benchmark}

The set of winning positions of $D$-heap nim is a closed subset of $\N^D$ with respect to the relative topology of $\Zp^D$ for the case $p = 2$, as its Grundy number is defined by the $D$-ary nimber sum $\N^D \to \N$, which is $2$-adically continuous. Since we already know how to compute the Grundy number of a given position of $D$-heap nim, we can demonstrate Algorithm \ref{defining function} to check the performance by applying it to the set of winning positions of $D$-heap nim with piles smaller than $M$. If there is another impartial game whose winning positions are characterised by a $p$-adically continuous function, then it will also provide similar benchmark tasks.

\vs
We experimented in two settings for the case $(D,E,M) = (3,10,100)$:
\bi
\item[(1)] Estimate whether $\vec{x}$ is a winning position or not for random $\vec{x}$ chosen from $\N_{<2^E}^D = \N_{<1024}^3$.
\item[(2)] Estimate whether $\vec{x}$ is a winning position or not for all $\vec{x}$ in $\ens{0} \times \N_{<2^E}^{D-1} = \ens{0} \times \N_{<1024}^2$. 
\item[(3)] Estimate whether $\vec{x}$ is a winning position or not for random winning position $\vec{x}$ chosen from $\N_{<2^E}^D = \N_{<1024}^3$.
\item[(4)] Estimate whether $\vec{x}$ is a winning position or not for all winning position $\vec{x}$ in $\N_{< 64} \times \N_{<2^E}^{D-1} = \N_{< 64} \times \N_{<1024}^2$. 
\ei
Here, we detect that $\vec{x}$ is a winning position if and only if the estimation of $F_Y(\vec{x})$ modulo $p^E$ is non-zero.

\vs
We note that the trivial detection, i.e.\ the method returning ``true'' for all input achieves $99.90 \%$ success for (1) and (2) because there is only $2^E = 1024$ losing positions in $\N_{< 2^E}^D = \N_{< 1024}^3$. Of course, it fails for all trials for (3) and (4). Therefore, performance of a model to solve $p$-adic manifold learning should be evaluated by the following two factors
\bi
\item[(i)] The success probability for positions including both winning positions and losing positions with the rate around $1023 : 1$.
\item[(ii)] The success probability for winning positions.
\ei
That is why we consider the tasks (1) -- (4) as a benchmark. The result for (i) should not be too smaller than the performance $99.90 \%$ of the trivial detection. The result for (ii) should not be too smaller than $10^{-1}$, because if too smaller, then any reasonable repetition of the method (with randomised parameters or changed hyper-parameters) does not achieve a high probability.

\vs
Here are the results of Algorithm \ref{defining function} for the tasks (1) -- (4).
\bi
\item[(1)] We observed $200$ failures in $10^5$ trials. The the $200$-th failure was observed at the $99668$-th trial, and the $201$-st failure was observed at the $100118$-th trial. Therefore, the success probability is estimated as $99.80 \%$.
\item[(2)] We observed $2112$ failures in $(2^E)^2 = 1048576$ trials. Therefore, the success probability is estimated as $99.79 \%$.
\item[(3)] We observed $43987$ failures in $5 \times 10^4$ trials. The the $43987$-th failure was observed at the $50000$-th trial, and the $43988$-th failure was observed at the $50001$-st trial. Therefore, the success probability is estimated as $12.02 \%$.
\item[(4)] We observed $48672$ failures in $64 \times 2^E = 65536$ trials. Therefore, the success probability is estimated as $25.73 \%$.
\ei
Thus, we achieved a high performance of the detection of a winning strategy for all tasks.

\vs
We note that the expected lowerbound
\be
p^{ED}  \times \frac{\# \left( Y \cap \N_{< 2^E}^D \right)}{\# \left( X \cap \N_{< 2^E}^D \right)} = 2^{ED} 2^{-E} = 2^{E(D-1)} = 2^{20} = 1048576
\ee
of $\# I$ for the success of the simple estimation in \S \ref{p-adic Interpolation of Zero Locus} is much larger than the actual value $7984$ of $\# I$ for the case $(D,M) = (3,100)$. Therefore, the success of the estimation implies that the $p$-adic interpolation effectively prevented overfitting the sparse data.

\section{Profinite Abelian Group Extension}
\label{Profinite Abelian Group Extension}

We explain extension of Algorithm \ref{defining function} to a profinite Abelian group using a $p$-adic analogue of Fourier transform. For this purpose, we explain the surprising fact that Mahler expansion can be viewed as a $p$-adic analogue of Fourier series expansion. This section includes arguments on formal scheme, and requires knowledge on arithmetic geometry.

\subsection{$p$-adic Fourier Series Expansion}
\label{p-adic Fourier Series Expansion}

We explain how to regard Mahler expansion as a $p$-adic analogue of Fourier series expansion. Let $\hGm$ denote the formal completion of the algebraic group $\Gm$ of the unit group over $\Zp$. It is an affine formal group scheme of the continuous character of $\Zp$.

\vs
Iwasawa isomorphism gives a canonical Hopf monoid isomorphism between its coordinate ring $\rH^0(\hGm,\cO_{\hGm}) \cong \Zp[[T-1]]$, where the comultiplication of $\Zp[[T-1]]$ is given by $T \mapsto T \otimes T$, and the Iwasawa algebra $\Zp[[\Zp]]$, and the image of the canonical multiplicative embedding $\Zp \hookrightarrow \Zp[[\Zp]]$ coincides with the group of characters of $\hGm$. In this sense, $(\Zp,\hGm)$ satisfies a $p$-adic analogue of Pontryagin duality, and Iwasawa isomorphism gives an interpretation of functions on $\hGm$ into topological linear combinations of characters of $\Zp$.

\vs
Amice transform gives a canonical isomorphism between the Schikhof dual of $\rC(\Zp,\Qp)$ and $\Zp[[\Zp]]$. The composite of Iwasawa isomorphism and Amice transform gives an isomorphism between the Schikhof dual of $\rC(\Zp,\Qp)$ and $\rH^0(\hGm,\cO_{\hGm})$, and the result is extended to a $p$-adic analogue of Plancherel's theorem (cf.\ \cite{Mih21} Theorem 3.23 and Theorem 3.24). In this sense, the composite is a $p$-adic analogue of Fourier series expansion.

\vs
In addition, the closed unit ball of the Schikhof dual of $\rH^0(\hGm,\cO_{\hGm}) \cong \Zp[[T-1]]$ is canonically isomorphic to the Banach $\Qp$-algebra $\rC_0(\N,\Qp)$ with the multiplication dual to the comultiplication of the Hopf object structure on $\Zp^{\N}$ identified with $\Zp[[T-1]] \cong \rH^0(\hGm,\cO_{\hGm})$ by the topological $\Zp$-linear basis $((T-1)^{\ell})_{\ell=0}^{\infty}$.

\vs
Therefore, the Schikhof dual of the $p$-adic Fourier series expansion gives a Banach $\Qp$ algebra isomorphism $\rC(\Zp,\Qp) \to \rC_0(\N,\Qp)$. This isomorphism assigns to each $f \in \rC(\Zp,\Qp)$ the sequence $(\int f d \mu_i)_{i=0}^{\infty}$, where $\mu_i \in \Zp[[\Zp]]$ is the $p$-adic measure corresponding to $(T-1)^i \in \Zp[[T-1]] \cong \rH^0(\hGm,\cO_{\hGm})$ through Iwasawa isomorphism for any $i \in \N$.

\vs
Let $i \in \N$. Since $\mu_i$ coincides with $([1] - [0])^i$ as an element of the group algebra $\Zp[\Zp] \subset \Zp[[\Zp]]$, we have
\be
\int f d \mu_i = \int f d ([1] - [0])^i = \int f d \sum_{j=0}^{i} (-1)^{i-j} \binom{i}{j} [j] = \sum_{j=0}^{i} (-1)^{i-j} \binom{i}{j} f(j).
\ee
Since the right hand side coincides with the $j$-th coefficient in Mahler's expansion of $f$, we conclude that Mahler expansion is also a $p$-adic analogue of Fourier series expansion.

\subsection{$p$-adic Fourier Transformation}
\label{p-adic Fourier Transformation}

As we have explained in the last subsection, the isomorphism between the Schikhof dual of $\rC(\Zp,\Qp)$ and $\rH^0(\hGm,\cO_{\hGm})$ given as the composite of Iwasawa isomorphism and Amice transform is extended to a $p$-adic analogue of Plancherel's theorem, which should be regarded as a $p$-adic Fourier transformation.

\vs
More precisely, let $G$ be a profinite Abelian group. We denote by $\hat{G}$ the formal group scheme of characters of $G$. The $p$-adic analogue of Plancherel's theorem (cf.\ \cite{Mih21} Theorem 3.23) implies that there is a natural isomorphism between the Schikhof dual of $\rC(G,\Qp)$ and $\rH^0(\hat{G},\cO_{\hat{G}})$. By \cite{SGA3-1} Expose VIIB 0.3.8.\ Corollaire, the compact Hausdorff linear topological $\Zp$-module admits an $\Zp$-linear homeomorphism to a direct product of copies of $\Zp$. Choosing such a homeomorphism is equivalent to choosing a orthonormal Schauder basis of $\rC(G,\Qp)$ by Schikhof duality.

\vs
For the case $G = \Zp$, we have chosen the homeomorphism $\rH^0(\hGm,\cO_{\hGm}) \cong \Zp[[x-1]] \cong \Zp^{\N}$ by the expansion by $x-1$. Similarly, the profinite group structure of $G$ sometimes gives an explicit choice of a homeomorphism. For example, if $\hat{G}$ admits a natural structure as a subgroup of a general linear group over $\Qp$, the coordinate centred at the unit might give a homeomorphism. If we fix one homeomorphism, then the corresponding orthonormal Schauder basis, which is a generalisation of Mahler basis, works for finite rank approximation.

\vs
In this way, arguments in \S \ref{p-adic Finite Rank Approximation of Defining Function} extends to $G$, and Algorithm \ref{defining function} extends to the $p$-adic manifold learning for $(G,Y)$ for a $\CO_{\delta}$ set $Y$ of $G$, as long as we fix an implementation of a data structure for a set of elements of $G$ like Algorithm \ref{kd-trie}.

% \newpage
\vspace{0.3in}
\addcontentsline{toc}{section}{Acknowledgements}
\noindent {\Large \bf Acknowledgements}
\vspace{0.2in}

\noindent
I thank all people who helped me to learn mathematics and programming. I also thank my family.

%\vspace{0.3in}
%\addcontentsline{toc}{section}{Compliance with Ethical Standards}
%\noindent {\Large \bf Compliance with Ethical Standards}
%\vspace{0.2in}
%
%\begin{description}
%\item[Funding:] We declare that there exists no funding supporting us.
%\item[Conflict of Interest:] We declare that there exists no conflict of interest.
%\item[Author Contribution:] We declare that we substantially contributed to the entire study, including the conception and design of the study, acquisition of data if applicable, analysis and interpretation of data if applicable, drafting and revising the article critically for important intellectual content, and final approval of the version to be submitted.
%\item[Data Availability Statements:] We declare that we analysed no existing data set because our research proceeds a theoretic and mathematical approach. We simply input random numbers generated by ``randint'' method of ``random'' module in python standard library during the execution of an experiment code.
%\item[Duplicated Submission:] We declare that the manuscript, including related data if applicable, figures if applicable, and tables if applicable, is not and will not be under active consideration elsewhere for the duration of the review process, and has never been published in a journal or presented in a poster session. 
%\end{description}

\end{document}